\documentclass[letterpaper,10pt,conference]{ieeeconf}
\IEEEoverridecommandlockouts       
\usepackage{amssymb}
\usepackage{amsmath}

\usepackage{amsthm}
\usepackage[english]{babel}
\usepackage{float}
\usepackage{graphicx}
\usepackage{hyperref}
\usepackage{mathtools}
\usepackage{booktabs}
\usepackage{afterpage}
\usepackage{dblfloatfix}
\usepackage{colortbl}
\usepackage{tikz}

\usepackage{filecontents}
\usepackage[noadjust]{cite}
\usepackage[font=footnotesize,labelfont=footnotesize]{subcaption}
\usepackage[font=footnotesize,labelfont=footnotesize]{caption}
\usepackage{comment}
\usepackage{authblk}
\usepackage{multirow}
\usepackage{xcolor}
\usepackage{algorithm}
\usepackage{algorithmic}
\usepackage{xspace}
\usepackage{graphicx}
\usepackage[english]{babel}
\usepackage[export]{adjustbox}

\usepackage{subcaption}
\definecolor{darkblue}{rgb}{0.15,0.15,0.55}
\definecolor{lightgrey}{rgb}{0.75,0.75,0.75}

\newtheorem{theorem}{Theorem}

\theoremstyle{definition}

\DeclareRobustCommand{\stepnum}[1]{%
  \tikz[baseline=(char.base)]{
    \node[
      circle,
      fill=black,
      text=white,
      inner sep=0.7pt,
      font=\scriptsize\bfseries
    ] (char) {#1};
  }%
}

\begin{document}
\title{\LARGE \bf Capability-Conditioned Graph Control Barrier Functions for Zero-Shot Safe Navigation of Heterogeneous Mobile Robot Teams}

\author{
Joonkyung Kim$^1$, Janghyun Cho$^2$, Jeehyun Lee$^3$,
Changjoo Nam$^3$, Yiwei Lyu$^{1,*}$
\thanks{
$^1$Dept. of Computer Science and Engineering, Texas A\&M University.
$^2$Dept. of Artificial Intelligence, Sogang University.
$^3$Dept. of Electronic Engineering, Sogang University.
$^*$Corresponding author: {\tt\small yiweilyu@tamu.edu}
}
}


\maketitle 

\begin{abstract}

Multi-robot systems are increasingly deployed as heterogeneous fleets
with diverse kinematic and actuation capabilities. While graph control
barrier function (GCBF) methods provide a scalable framework for
decentralized safe control, existing approaches typically assume
homogeneous robot dynamics. Methods that directly encode robot classes
can distinguish known types, but do not represent state-dependent motion
capabilities or naturally transfer to unseen robot models.
To address this challenge, we propose the
\textit{Capability-Conditioned Heterogeneous Graph Control Barrier
Function (CC-HGCBF)}, defined in a shared second-order reference-point
operational space. This formulation provides a common geometric
representation of safety across heterogeneous robots while retaining
their model-specific dynamics and admissible control constraints.
To enable zero-shot transfer, we introduce the Radial Capability Field
(RCF), a compact, fixed-dimensional representation of each robot's
state-dependent finite-horizon reachable velocity set. Conditioning the
shared graph barrier and decentralized policy on the RCF enables them to
adapt to changes in robot state, kinematic constraints, and control
limits without relying on predefined robot identities.
We jointly train the barrier and policy using proactive safety labeling
and neural CBF-QP supervision, and evaluate them on mixed teams containing
both seen and unseen robot models. Experiments demonstrate scalable safe
coordination and zero-shot transfer to robot models and capability
variations withheld during training. Physical-robot experiments further
demonstrate zero-shot deployment to unseen kinematic models without
retraining.

\end{abstract} 

\section{Introduction}
\label{sec:intro}
\vspace{-6pt}

Multi-robot systems are increasingly expected to operate as evolving
fleets rather than fixed groups of identical platforms~\cite{rizk2019coop_hetero_mr}. Robots may be replaced, new platforms introduced after deployment, and
a single mission may involve robots with different mobility and
actuation capabilities. This heterogeneity is particularly important
for safety-critical navigation. Robots with different kinematic
structures can respond very differently to the same interaction: an
omnidirectional robot may move laterally to avoid a collision, whereas
a car-like robot is constrained by its heading and steering
angle~\cite{alonso2012reciprocal, alonso2013optimal}.
 Even at identical positions and velocities, their
avoidance capabilities may differ due to dynamics, control limits, and
current operating states. Safe coordination therefore requires
controllers that can operate across heterogeneous platforms and remain
effective when previously unseen robots join the team~\cite{howell2023general_hetero_mr, bohlinger2025one_policy}.

\begin{figure}[!t]
    \captionsetup{skip=1pt}
    \centering
    \includegraphics[width=0.48\textwidth]{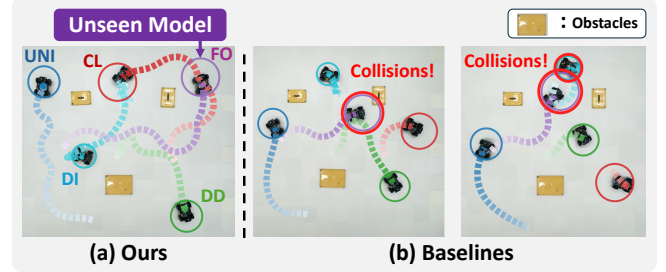}
    \caption{Qualitative trajectory comparison of zero-shot policy transfer to an unseen robot model in a $3.8\,\mathrm{m}\times3.8\,\mathrm{m}$ workspace, where five robots navigate from randomly sampled starts to goals while avoiding three static obstacles (yellow boxes). A forward-only (FO) robot (purple), with motion capabilities different from those seen during training, is introduced only at evaluation. \textbf{(a)} Our capability-conditioned policy maintains safe trajectories with the unseen FO robot, while \textbf{(b)} baselines relying on discrete robot identities result in robot--robot or robot--obstacle collisions.}
    \label{fig:qualitative} \vspace{-2em}
\end{figure}

Graph control barrier functions (GCBFs)~\cite{zhang2023gcbf,zhang2025gcbf+}
provide scalable decentralized safety by learning graph-structured barriers
and policies from local observations. However, existing GCBF formulations
typically assume common robot dynamics. Their graph representations capture
neighboring states and relative motion, but do not explicitly represent how
each neighbor's dynamics constrain its future motion. Recent heterogeneous
GCBF methods address differences in state and action spaces through
zero-padding for feature alignment and robot-type embeddings for model
identification~\cite{yang2026hgcbf}. While such representations can
distinguish robot types observed during training, they do not describe the
motion capability available at the current state or naturally represent an
unseen model~\cite{howell2023general_hetero_mr}. The central challenge
therefore remains: \textit{How can a single distributed graph barrier reason about
heterogeneous robot motion capabilities and transfer to unseen robot
models without retraining?}
{
To answer this question, we propose a \textit{Capability-Conditioned
Heterogeneous GCBF (CC-HGCBF)} framework for safe learning, where a
shared graph barrier and decentralized policy are conditioned on
robot-specific motion capabilities. 
}{This enables a single learned controller to coordinate mixed teams and
transfer to previously unseen supported models without retraining. Specifically, (1) we propose CC-HGCBF,
a novel heterogeneous GCBF formulation in a shared 
operational space, and establish its safety condition for mixed robot
teams. This formulation provides a common representation of relative motion and safety
while preserving model-specific responses and control constraints; (2) We introduce a state-dependent capability representation,
the Radial Capability Field (RCF), which encodes each robot's
finite-horizon reachable velocity set in a fixed-dimensional form.
Conditioning the shared barrier and policy on the RCF enables zero-shot
transfer to unseen supported models without predefined model identities
or retraining; and (3) We evaluate the scalability and zero-shot generalization of
CC-HGCBF across increasing team sizes and obstacle densities, including
robot models and capability parameters withheld during training.
Comparisons with capability-agnostic, type-conditioned, and
static-capability baselines highlight the benefit of state-dependent
capability conditioning, while physical-robot experiments validate its
zero-shot deployment. 

\section{Related Work} 
\label{sec:relWork}





Control barrier functions (CBFs) provide a principled framework for
enforcing safety through forward invariance of a prescribed safe
set~\cite{ames2016cbf_qp,ames2019cbf}. In multi-robot systems,
analytical CBF methods commonly construct pairwise safety constraints
and enforce them through online quadratic programs
(QPs)~\cite{wang2017mr_cbf}. However, the number of constraints grows
with team size and interaction density, increasing computational cost,
while input limits and simultaneous safety requirements can compromise
QP feasibility~\cite{tan2022compatibility}.


Learning-based graph CBFs improve scalability by aggregating local
interactions through shared graph neural networks
(GNNs)~\cite{qinlearning,dawson2023learn_certificates_survey}.
GCBF learns a graph-structured barrier and decentralized
policy~\cite{zhang2023gcbf}, while GCBF+ introduces neural CBF-QP
supervision and incorporates actuation limits~\cite{zhang2025gcbf+}.
HJB-GNN further combines graph CBF and value-function learning through a
constrained Hamilton--Jacobi--Bellman formulation~\cite{wang2025HJB-GNN}.
Although these methods generalize across team sizes and interaction
topologies, they are primarily formulated for robots with common
dynamics and do not explicitly represent the state-dependent motion
capabilities that determine how heterogeneous robots can respond.

Analytical CBF methods for heterogeneous teams explicitly account for
differences in robot dynamics and control limits. Multi-robot barrier
certificates use model-specific pairwise constraints
~\cite{wang2016hetero_mr_cbf}, Parametric-CBF methods characterize
heterogeneous safety through estimated barrier parameters
~\cite{lyu2022parametric}, and capability-aware CBFs allocate
collision-avoidance responsibility according to robot motion
capabilities~\cite{kim2026capability}. These approaches retain
model-specific dynamics, but generally rely on pairwise modeling,
explicit responsibility allocation, or online optimization tailored to
the participating robot models, motivating scalable approaches that can
accommodate heterogeneous motion capabilities.

Recent learning-based GCBF methods address heterogeneity through aligned
state and action representations. In particular, hGCBF uses zero-padding
and one-hot model-type embeddings to process multiple robot types with a
shared GNN~\cite{yang2026hgcbf}. Such type-conditioned representations
distinguish robot classes observed during training, but do not describe
how their feasible responses vary with kinematic constraints, actuation
limits, or current state, nor do they provide a natural representation
for capabilities unseen during training
~\cite{howell2023general_hetero_mr,fu2025cash}.
Taken together, these gaps motivate capability-conditioned graph barriers
that retain the scalability of GCBFs while representing heterogeneous
interactions in a shared operational space and conditioning decisions on
state-dependent motion capabilities, including for robot models unseen
during training.

\begin{figure*}[t]
    \centering

    \begin{minipage}[t]{0.70\textwidth}
        \vspace{0pt}
        \centering
        \includegraphics[
            width=\linewidth,
            trim={0 0 0 0},
            clip
        ]{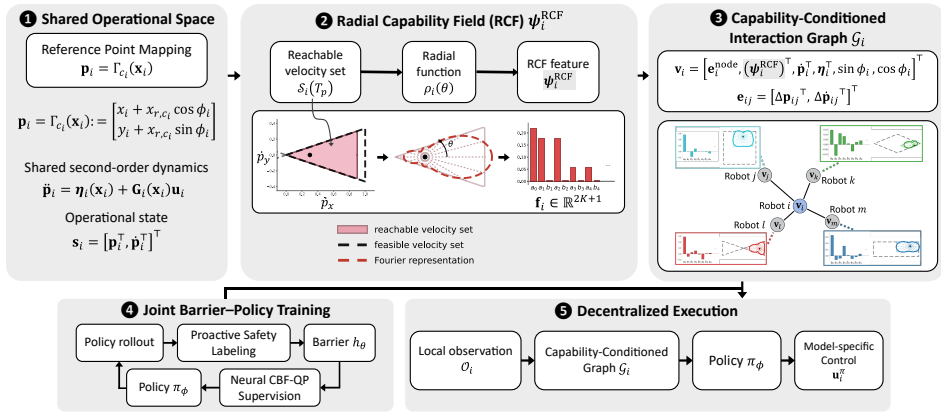}
    \end{minipage}
    \hfill
    \begin{minipage}[t]{0.285\textwidth}
        \vspace{0pt}
        \caption{
    Overview of the proposed CC-HGCBF framework.
    \stepnum{1} Heterogeneous robot motion is mapped into a shared
    reference-point operational space.
    \stepnum{2} The RCF represents each robot's state-dependent
    finite-horizon motion capability in this common space.
    \stepnum{3} These capability features are embedded into the local
    interaction graph and incorporated into neighboring robots' safety reasoning
    through graph message passing.
    \stepnum{4} The shared barrier and decentralized policy are
    jointly trained using proactive safety labeling and neural CBF-QP supervision.
    \stepnum{5} During decentralized execution, each robot constructs
    its local capability-conditioned graph and applies the learned policy through
    its model-specific control interface.
    }
        \label{fig:overview}
    \end{minipage}

    \vspace{-1.5em}
\end{figure*}

\section{Problem Statement} \vspace{-3pt}
\label{sec:problem_statement}

{

We consider a team of $N$ mobile robots operating in a shared 2D
workspace, indexed by $\mathcal V=\{1,\ldots,N\}$. Each robot $i$ has a
model $c_i\in\mathcal C$, where $\mathcal C$ denotes a family of
heterogeneous mobile-robot models. Its dynamics are described by the
model-dependent control-affine system $\dot{\mathbf x}_i = \mathbf f_{c_i}(\mathbf x_i) + \mathbf g_{c_i}(\mathbf x_i)\mathbf u_i, \; \mathbf u_i\in\mathcal U_{c_i}(\mathbf x_i)$, 
where $\mathbf x_i$ and $\mathbf u_i$ denote the state and control input,
and $\mathcal U_{c_i}(\mathbf x_i)$ is the admissible control set.

We assume that each robot model admits a reference-point mapping
$\mathbf p_i=\Gamma_{c_i}(\mathbf x_i)$ such that its motion can be
expressed in the common second-order operational form \vspace{-3pt}
\begin{equation}
    \ddot{\mathbf p}_i
    =
    \boldsymbol{\eta}_{c_i}(\mathbf x_i)
    +
    \mathbf G_{c_i}(\mathbf x_i)\mathbf u_i.
    \label{eq:operational_dynamics}
\end{equation}
Here, $\mathbf p_i\in\mathbb R^2$ denotes the operational position, while $\boldsymbol{\eta}_{c_i}$ and $\mathbf G_{c_i}$ preserve the model-dependent motion response within the shared representation. The construction of this representation is introduced in Sec.~\ref{sec:method}.
}

{

The robots operate using only local observations. Robot $i$ interacts
with neighbors within sensing radius $R$, with neighborhood $\mathcal{N}_i = \left\{j \in \mathcal{V} \setminus \{i\}\;\middle|\; \lVert \mathbf{p}_i - \mathbf{p}_j \rVert \leq R\right\}$.
Static obstacles are perceived using a 2D LiDAR with
$N_{\mathrm{ray}}$ uniformly distributed rays over the same range. The
local observation $\mathcal O_i$ contains neighboring-robot information
and LiDAR measurements, without access to global information.

Each robot is assigned a goal $\mathbf p_i^g\in\mathbb R^2$ and an
operational safety radius $r_i>0$. Safe operation requires
$\lVert\mathbf p_i-\mathbf p_j\rVert>r_i+r_j$ for all $i\neq j$, with
an analogous separation constraint for static obstacles.
During training, the controller is exposed only to robot models in
$\mathcal C_{\mathrm{train}}\subset\mathcal C$. At deployment, models in
$\mathcal C\setminus\mathcal C_{\mathrm{train}}$ may be introduced
without retraining. \textbf{Our objective} is to learn a shared decentralized
policy
$\pi:\mathcal O_i\rightarrow\mathbf u_i$,
with $\mathbf u_i\in\mathcal U_{c_i}(\mathbf x_i)$, for safe
goal-directed navigation across varying team sizes and robot models,
including models unseen during training. \vspace{-6pt}

\vspace{-0.2em}
\section{Method}
\label{sec:method}


Designing a shared safety controller for heterogeneous robot teams requires
resolving two closely connected questions. First, how should different
robots be represented so that their interactions are directly comparable
while their distinct abilities to respond remain visible? Second, once
such a representation is available, how can it support scalable
decentralized safety learning that remains effective for robot models not
encountered during training? The following sections develop CC-HGCBF
around these two questions.



\subsection{Shared Operational-Space Representation}
\label{subsec:operational_space}

\subsubsection{Reference-Point Mapping}
We first specify the reference-point mapping assumed in
Sec.~\ref{sec:problem_statement}. Different robot models describe their
motion in model-specific coordinates, making relative motion difficult to
represent consistently across heterogeneous robots
~\cite{yang2026hgcbf,dawson2023learn_certificates_survey}. We therefore
represent each robot by a reference point whose position and velocity are
expressed in common Cartesian coordinates while retaining its
model-specific dynamics and control constraints
~\cite{lynch2017modern_robotics}.


{
For robot $i$, let $\bar{\mathbf p}_i=[p^{\mathrm b}_{i,x},p^{\mathrm b}_{i,y}]^\top$ denote its body-center position.
We define the model-dependent reference-point mapping
$\Gamma_{c_i}:\mathcal X_{c_i}\rightarrow\mathbb R^2$ as \vspace{-6pt}
\begin{equation}
\mathbf p_i
=
\Gamma_{c_i}(\mathbf x_i)
:=
\bar{\mathbf p}_i
+
x_{r,c_i}
\begin{bmatrix}
\cos\phi_i\\
\sin\phi_i
\end{bmatrix},
\label{eq:reference_point}
\end{equation}
where $\phi_i$ is the robot heading and $x_{r,c_i}$ is the
model-dependent reference-point offset. For holonomic models,
$x_{r,c_i}=0$, so the reference point coincides with the body center;
for nonholonomic models, a nonzero look-ahead distance ensures a
nonsingular input transformation.
}

At this reference point, longitudinal and angular motion induce a
two-dimensional Cartesian velocity, providing a common representation
across robot models. Differentiating~\eqref{eq:reference_point} twice
yields the second-order operational dynamics in
\eqref{eq:operational_dynamics}, where
$\boldsymbol{\eta}_{c_i}(\mathbf x_i)$ is the state-dependent drift and
$\mathbf G_{c_i}(\mathbf x_i)$ maps the model-specific control to
reference-point acceleration. Each robot is represented by the
common operational state
$\mathbf s_i=[\mathbf p_i^\top,\dot{\mathbf p}_i^\top]^\top$.
Consequently, $\mathbf p_i-\mathbf p_j$ and
$\dot{\mathbf p}_i-\dot{\mathbf p}_j$ have consistent physical meanings
across robot models and can be used as shared interaction features.
Model-specific responses are retained through
$\boldsymbol{\eta}_{c_i}$, $\mathbf G_{c_i}$, and
$\mathcal U_{c_i}(\mathbf x_i)$.



{
\subsubsection{Operational-Space Safety}
\label{subsubsec:op_safety}

Because the reference point may be offset from the body center, we define
the operational safety radius as $r_i:=r_i^{\mathrm{phys}}+x_{r,c_i}$, where
$r_i^{\mathrm{phys}}$ bounds the robot's physical extent from its body
center.
Since $\|\mathbf p_i-\bar{\mathbf p}_i\|=x_{r,c_i}$, the triangle inequality
implies $\|\bar{\mathbf p}_i-\bar{\mathbf p}_j\|
\geq
\|\mathbf p_i-\mathbf p_j\|-x_{r,c_i}-x_{r,c_j}.$
Thus, $\|\mathbf p_i-\mathbf p_j\|\geq r_i+r_j$ guarantees
$\|\bar{\mathbf p}_i-\bar{\mathbf p}_j\|
\geq r_i^{\mathrm{phys}}+r_j^{\mathrm{phys}}$, which is sufficient to
prevent overlap of their physical bounding regions. 
\vspace{-3pt}

}




\subsection{Radial Capability Field}
\label{subsec:rcf}

The shared operational representation makes heterogeneous robot motion
directly comparable, but does not describe how each robot can respond to
an interaction. Two robots may have the same operational velocity while
having very different abilities to change it: one may accelerate
laterally, whereas another may change direction only through forward
motion. These differences are critical for collision avoidance and depend
on the robot model, current state, and admissible controls.
We therefore introduce the RCF, a fixed-dimensional representation of each robot's state-dependent
finite-horizon motion capability in the shared operational-velocity
space. RCF characterizes the directional velocity changes reachable from
the current state by constructing a local reachable velocity set, mapping
it into the shared space, and encoding its directional extent.



\subsubsection{Finite-Horizon Reachable Velocity Approximation}

We first distinguish between the velocities permitted by a robot model and those that are locally reachable from the robot's current state.
For a nonholonomic robot, let $\mathbf z_i=[v_i,\omega_i]^\top$ denote its body-velocity state, where
$v_i$ and $\omega_i$ are the longitudinal and angular velocities,
respectively. For a holonomic robot, $\mathbf z_i$ denotes its
two-dimensional Cartesian velocity.


Let $\mathcal V_{c_i}$ denote the feasible velocity set determined by the
kinematic and velocity constraints of model $c_i$. Since
$\mathcal V_{c_i}$ describes all permitted velocities but not those
reachable from the current state within a short response time, we
approximate the finite-horizon reachable set as \vspace{-6pt}
%
\begin{equation}
\mathcal W_i(\mathbf x_i,T_p)
=
\left\{
\mathbf z_i + T_p\mathbf u
\mid
\mathbf u\in\mathcal U_{c_i}(\mathbf x_i)
\right\}
\cap
\mathcal V_{c_i}.
\label{eq:reachable_velocity}
\end{equation}
Equation~\eqref{eq:reachable_velocity} assumes constant control over
$T_p$ and retains only velocities satisfying the model-specific feasible
velocity constraints. Thus, $\mathcal V_{c_i}$ describes what the robot
model permits, whereas $\mathcal W_i$ describes what the robot can reach
from its current state over the response horizon.

Because $\mathcal W_i$ remains expressed in model-specific velocity
coordinates, we map it into the shared operational-velocity space. Let
$\Phi_{c_i}$ denote the robot-aligned velocity mapping. For
nonholonomic robots,
\begin{equation}
\Phi_{c_i}
\left(
\begin{bmatrix}
v\\
\omega
\end{bmatrix}
\right)
=
\begin{bmatrix}
v\\
x_{r,c_i}\omega
\end{bmatrix},
\label{eq:velocity_mapping}
\end{equation}
%
while $\Phi_{c_i}$ is the identity for holonomic Cartesian velocity.
The current operational velocity is
$\mathbf q_i=\Phi_{c_i}(\mathbf z_i)$, and the corresponding reachable
set is
$\mathcal S_i(T_p)
=
\Phi_{c_i}\!\left(\mathcal W_i(\mathbf x_i,T_p)\right).$
Thus, $\mathcal S_i(T_p)\subset\mathbb R^2$ has common physical velocity
coordinates across robot models while retaining the effects of current
state, kinematic constraints, and admissible controls.

\subsubsection{State-Centered Radial Encoding}

The reachable sets $\mathcal S_i(T_p)$ share a common space but can have
different geometries across robot models and states. To obtain a
fixed-dimensional representation, we characterize each set by its
directional extent from the current operational velocity $\mathbf q_i$,
thereby encoding available velocity changes rather than absolute feasible
velocities.


Let
$\mathbf d(\theta)=[\cos\theta,\sin\theta]^\top$,
$\theta\in[0,2\pi)$, denote a unit direction. We define
\begin{equation}
\rho_i(\theta)
=
\max
\left\{
r\geq0
\;\middle|\;
\mathbf q_i+\tau\mathbf d(\theta)\in\mathcal S_i(T_p),
\ \forall\tau\in[0,r]
\right\}.
\label{eq:radial_function}
\end{equation}
Here, $\rho_i(\theta)$ measures the maximum velocity change available
from the current state in direction $\theta$. Large values indicate
greater response capability, while small or zero values indicate limited
motion. Directional asymmetries capture effects such as unequal forward
and backward mobility, and the profile varies with state and admissible
controls.


We evaluate $\rho_i$ at $N_\theta$ uniformly spaced directions,
$\theta_m=2\pi m/N_\theta$, and encode the resulting directional profile
with a truncated Fourier series of order $K$, \vspace{-6pt}
\begin{equation}
\rho_i(\theta)
\approx
a_{i,0}
+
\sum_{k=1}^{K}
\left(
a_{i,k}\cos(k\theta)
+
b_{i,k}\sin(k\theta)
\right),
\label{eq:fourier_rcf}
\end{equation}
where $a_{i,k}$ and $b_{i,k}$ are the corresponding Fourier coefficients. We collect the retained coefficients as
\begin{equation}
\mathbf f_i
=
[a_{i,0},a_{i,1},b_{i,1},\ldots,a_{i,K},b_{i,K}]^\top
\in\mathbb R^{2K+1}.
\label{eq:fourier_feature}
\end{equation}

The RCF feature is then defined as
\begin{equation}
\boldsymbol\psi_i^{\mathrm{RCF}}
=
\left[
\frac{x_{r,c_i}}
{\max_{c'}x_{r,c'}},
\mathbf f_i^\top
\right]^\top
\in\mathbb R^{2K+2}.
\label{eq:rcf_feature}
\end{equation}
We choose $2K+1\leq N_\theta$ with $N_\theta$ sufficiently large to
resolve the directional profile. The RCF construction and example
state-dependent representations are shown in
Fig.~\ref{fig:overview}\stepnum{2} and ~\ref{fig:rcf_velocity_sets},
respectively.
Unlike type embeddings or static feasible-velocity representation, RCF
captures what a robot can do from
its current state over a finite
response horizon.  For any model admitting the shared operational
representation, RCF can be constructed independently of the classes used
during training, allowing unseen models to use the same graph input
without new type labels or architectural changes.
\subsection{Capability-Conditioned Heterogeneous GCBF}
\label{subsec:cchgcbf}

The shared operational space provides a common representation of relative
motion, while the RCF describes each robot's state-dependent response
capability. We combine these representations in a capability-conditioned
interaction graph, where edges encode relative operational states and
robot nodes encode motion capabilities. This allows a shared graph barrier
and policy to reason over common interaction geometry while accounting for
the capabilities of the participating robots, as illustrated in
Fig.~\ref{fig:overview}\stepnum{3}.



\subsubsection{Capability-Conditioned Interaction Graph}

At each control step, local interactions are represented by a directed
graph $\mathcal G=(\mathcal V_{\mathcal G},\mathcal E)$ containing robot,
goal, and locally observed obstacle nodes. A robot--robot edge
$j\rightarrow i$ is created when robot $j$ lies within sensing radius $R$,
with
$\mathcal E_{\mathrm{rr}}
=
\{(j,i)\mid j\neq i,\,
\|\mathbf p_j-\mathbf p_i\|_2\le R\}$.
Each robot is also connected to its goal, while active LiDAR returns are
represented as static obstacle nodes. The graph is reconstructed as the
local interaction topology changes.

For robot $i$, the node feature is
{\small
\begin{equation}
\mathbf v_i
=
\left[
(\mathbf e_i^{\mathrm{node}})^\top,\,
(\boldsymbol{\psi}_i^{\mathrm{RCF}})^\top,\,
\dot{\mathbf p}_i^\top,\,
\boldsymbol{\eta}_{c_i}(\mathbf x_i)^\top,\,
\sin\phi_i,\,
\cos\phi_i
\right]^\top,
\label{eq:agent_node_feature}
\end{equation}
}
where $\mathbf e_i^{\mathrm{node}}$ indicates the graph entity type,
$\boldsymbol{\psi}_i^{\mathrm{RCF}}$ is the RCF feature, and
$\boldsymbol{\eta}_{c_i}$ is the state-dependent acceleration term of
the operational dynamics. The node-type indicator distinguishes robots,
goals, and obstacles without encoding the robot model; robot-specific
entries are set to zero for goal and obstacle nodes.

Each directed edge carries the relative operational state \vspace{-4pt}
\begin{equation}
\mathbf e_{ij}
=
\left[
(\mathbf p_j-\mathbf p_i)^\top,\,
(\dot{\mathbf p}_j-\dot{\mathbf p}_i)^\top
\right]^\top .
\label{eq:edge_feature}
\end{equation}
Goal and obstacle edges are defined analogously, with zero velocity for
static nodes. Thus, edge features describe \emph{what interaction is
occurring}, while the RCF describes \emph{how each robot can respond}.
This separation provides a common graph representation without requiring
discrete robot identities.

We use the attention-based message-passing architecture of
GCBF+~\cite{zhang2025gcbf+} for the graph barrier and decentralized
policy, with independent network parameters shared across robots and
graph sizes. Message passing incorporates neighboring RCF features into
the local interaction representation, allowing both networks to account
for neighboring robots' current motion capabilities.


\begin{figure*}[!t]
    \centering
    \begin{subfigure}[b]{0.25\textwidth}
        \centering
        \includegraphics[width=0.8\linewidth]{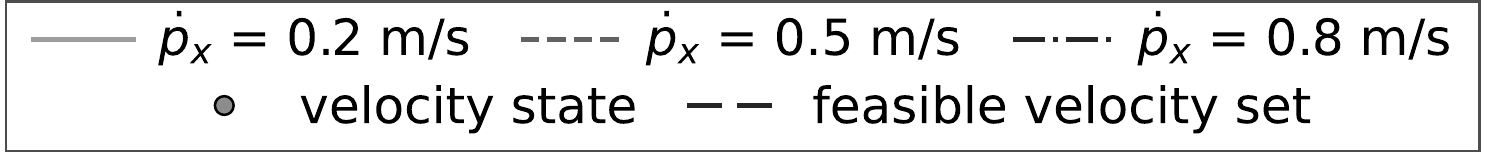}
        \par\vspace{0.3em}
        \hspace*{-4mm}%
        \includegraphics[width=0.8\linewidth]{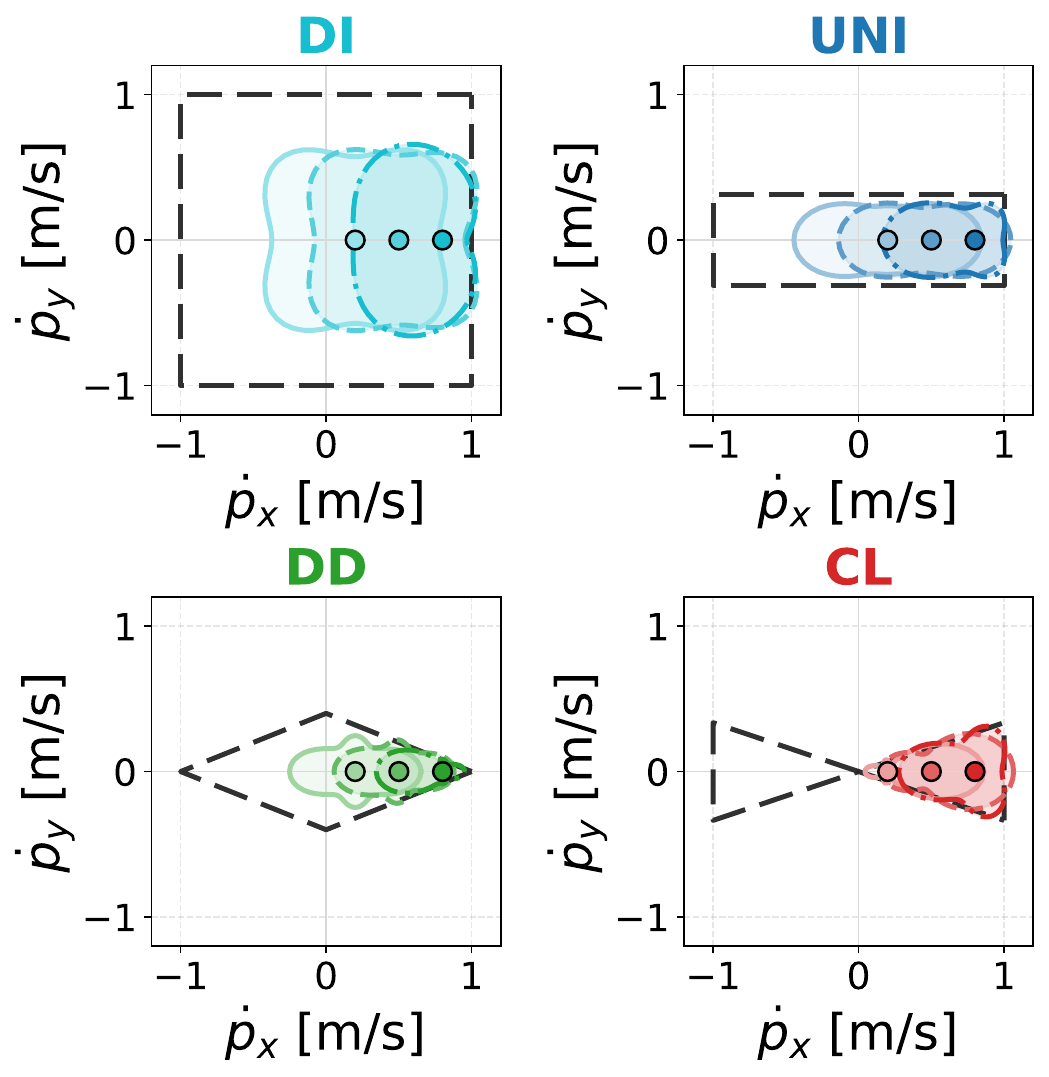}
        \caption{State-dependent RCF examples.}
        \label{fig:rcf_velocity_sets}
    \end{subfigure}%
    \hfill
    \begin{subfigure}[b]{0.75\textwidth}
        \centering
        \includegraphics[width=\linewidth]{Figures/training_v2.pdf}
        \par\vspace{0.5em}
        \caption{Training pipeline.}
        \label{fig:training}
    \end{subfigure}

    \caption{
    Capability representation and training procedure.
    \textbf{(a)} State-dependent RCF representations for DI, UNI, DD,
    and CL in the shared operational-velocity
    space~\cite{lynch2017modern_robotics}. Colored regions show the {$K=8$}
    RCFs at $v_i\in\{0.2,0.5,0.8\}$\,m/s with $\omega_i=0$.
    Circular markers denote the current operational velocities, and
    black dashed boundaries denote the model-specific feasible velocity
    sets mapped into the shared space.
    \textbf{(b)} Joint barrier--policy training
    pipeline~\cite{zhang2025gcbf+} with proactive safety labeling.
    The data collection example visualizes the learned barrier
    $h_\theta$ around an ego robot as a local safety map, with lower
    barrier values near detected obstacles and along the motion
    directions of neighboring robots. Policy rollouts generate
    interaction data, proactive safety labels supervise the graph
    barrier, and the neural CBF-QP provides supervisory actions for
    policy learning.
    }
    \label{fig:method_overview}
    \vspace{-1.5em}
\end{figure*}

\subsubsection{CC-HGCBF Formulation and Safety Condition}

For robot $i$, let $\mathcal G_i$ denote its local interaction graph and
$h_{\theta,i}:=h_\theta(\mathcal G_i)$ the corresponding barrier value.
The induced graph safe set is \vspace{-3pt}
\begin{equation}
\mathcal C_\theta
=
\left\{
\mathcal G
\;\middle|\;
h_{\theta,i}\ge0,\;
\forall i\in\mathcal V
\right\}.
\label{eq:learned_safe_set}
\end{equation}
Let $\mathcal U(\mathbf x):=\prod_{j\in\mathcal V}
\mathcal U_{c_j}(\mathbf x_j)$ denote the joint admissible control set,
$\alpha$ a locally Lipschitz extended class-$\mathcal K$
function, and
$\dot h_{\theta,i}:=\frac{d}{dt}h_\theta(\mathcal G_i(t))$ the
exact barrier derivative along a closed-loop trajectory.
Although the barrier parameters are shared across robot models, the
constraint is evaluated in the common operational space using
the model-specific dynamics and admissible control set of each robot
influencing the local graph. We refer to this shared
capability-conditioned barrier construction as the
\textbf{Capability-Conditioned Heterogeneous Graph Control Barrier
Function (CC-HGCBF)}.  The following CBF condition characterizes continuous-time feedbacks that preserve the learned safe set.

\begin{theorem}[Conditional Continuous-Time Invariance]
\label{thm:cchgcbf_invariance}
Let $\mathcal C_{\mathrm{op}}$ denote the operational collision-free
set defined by the separation constraints in Sec.~\ref{subsubsec:op_safety}. Suppose that $\mathcal C_\theta\subseteq\mathcal C_{\mathrm{op}}$ and that
$\mathcal G(0)\in\mathcal C_\theta$. Assume that
$h_{\theta,i}(\mathcal G_i(t))$ is continuously differentiable along
the closed-loop trajectory. If an admissible continuous-time feedback satisfies
$\mathbf u(t)\in\mathcal U(\mathbf x(t))$ and
$\dot h_{\theta,i}+\alpha(h_{\theta,i})\ge0$ for all
$i\in\mathcal V$ and $t\ge0$,
then the resulting trajectory satisfies
$\mathcal G(t)\in\mathcal C_\theta$ for all $t\ge0$.
Consequently, the operational separation constraints remain satisfied,
and the physical robot bodies remain collision-free.
\end{theorem}

\begin{proof}
By assumption, each barrier value satisfies
the standard CBF condition along the trajectory. Since
$h_{\theta,i}(\mathcal G_i(0))\ge0$ for all $i$, the standard CBF
forward-invariance argument~\cite{ames2019cbf} yields
$h_{\theta,i}(\mathcal G_i(t))\ge0$ for all $i$ and $t\ge0$.
Hence, $\mathcal G(t)\in\mathcal C_\theta$ for all $t\ge0$.
Because $\mathcal C_\theta\subseteq\mathcal C_{\mathrm{op}}$, the
operational separation constraints and, by
Sec.~\ref{subsubsec:op_safety}, physical-body separation are preserved.
\end{proof}
During training, the RCF and other state-dependent features condition
the barrier evaluation and are updated as the system evolves. At each
control update, the supervisory CBF-QP uses the CC-HGCBF
condition to generate targets for the decentralized policy
in Sec.~\ref{subsec:learning}.

\subsection{Learning the Shared Barrier and Decentralized Policy}
\label{subsec:learning}

We follow the joint barrier--policy training framework of
GCBF+~\cite{zhang2025gcbf+} under a centralized-training,
decentralized-execution (CTDE) paradigm. Policy rollouts generate
capability-conditioned interaction data, proactive safety labels supervise
the shared barrier, and a neural CBF-QP provides supervisory actions for
the decentralized policy, as illustrated in Fig.~\ref{fig:training}.

\subsubsection{Data Generation and Proactive Safety Labeling}
\label{subsec:safety_labeling}
Training data are generated on-policy from heterogeneous robot teams.
At each control step, local graphs are reconstructed as described in
Sec.~\ref{subsec:cchgcbf} from current robot and obstacle observations.

Direct geometric collision labels depend only on the current separation
and therefore identify unsafe interactions only when the safety boundary
is reached or violated. However, pairs at the same separation can have
very different safety implications depending on their relative motion:
one may be moving apart while another is rapidly closing toward
collision~\cite{fiorini1998vo}. We therefore construct proactive safety
labels using the first control-independent HOCBF layer~\cite{xiao2019hocbf}. In the shared operational space, this criterion
captures both relative position and velocity, identifying closing
interactions before geometric overlap. Because it is independent of the
control input, the same labeling criterion applies across heterogeneous
robot models.

For a robot pair $(i,j)$, define
$\Delta\mathbf p_{ij}=\mathbf p_i-\mathbf p_j$,
$\Delta\dot{\mathbf p}_{ij}=\dot{\mathbf p}_i-\dot{\mathbf p}_j$, and
$R_{ij}=r_i+r_j$. The pairwise geometric safety function is \vspace{-3pt}
\begin{equation}
\psi_{0,ij}
=
\|\Delta\mathbf p_{ij}\|_2^2-R_{ij}^2 .
\label{eq:cbf_psi0}
\end{equation}
Since the operational dynamics are second order,
$\dot\psi_{0,ij}$ is control independent. We define the first HOCBF layer
as
\begin{equation}
\psi_{1,ij}
=
2\Delta\mathbf p_{ij}^{\top}
\Delta\dot{\mathbf p}_{ij}
+
\lambda\psi_{0,ij},
\label{eq:cbf_psi1}
\end{equation}
where $\lambda>0$ is shared across robot models.
A negative value $\psi_{1,ij}<0$ indicates that the pair is approaching
the safety boundary too rapidly relative to its current separation.
We use this as a proactive unsafe training signal rather than
as a standalone safety certificate.

The same construction is applied to static obstacles. For an active
LiDAR return $\mathbf q_{ik}$ observed by robot $i$,
$\psi_{0,ik}=\|\mathbf p_i-\mathbf q_{ik}\|_2^2-r_i^2$ and
$\psi_{1,ik}
=
2(\mathbf p_i-\mathbf q_{ik})^\top\dot{\mathbf p}_i
+\lambda\psi_{0,ik}$.
The instantaneous unsafe label is then
{\small
\begin{equation}
\begin{aligned}
y_i^{\mathrm{unsafe}}=1
\ \Longleftrightarrow \ &
\exists j:
\bigl(
\psi_{0,ij}<0
\;\vee\;
\psi_{1,ij}<0
\bigr)
\\
&\vee\
\exists k:
\bigl(
\psi_{0,ik}<0
\;\vee\;
\psi_{1,ik}<0
\bigr).
\end{aligned}
\label{eq:unsafe_label}
\end{equation}}
Thus, robot--robot and robot--obstacle interactions use the same
position--velocity criterion in the shared operational space. 
Safe samples are constructed by aggregating the instantaneous labels
over a backward horizon $H$, following the temporal labeling convention
in~\cite{zhang2025gcbf+, wang2025HJB-GNN}; samples assigned to neither set receive no
barrier-sign supervision.



\subsubsection{Joint Barrier--Policy Training}
\label{subsec:joint_training}

We jointly train the CC-HGCBF $h_\theta$ and decentralized policy $\pi_\phi$ on the same
capability-conditioned graph. The networks have independent parameters,
and the policy predicts a residual correction to the nominal control (see Fig.~\ref{fig:training}).
During training, a neural CBF-QP provides supervisory actions for policy
learning.

For the joint control vector
$\mathbf u=[\mathbf u_1^\top,\ldots,\mathbf u_N^\top]^\top$,
the supervisory action is obtained by
{\begin{equation}
\begin{aligned}
\mathbf u^{\mathrm{qp}}
=
\arg\min_{\mathbf u}\quad&
\sum_{i\in\mathcal V}
\|\mathbf u_i-\mathbf u_i^{\mathrm{ref}}\|_2^2
\\
\mathrm{s.t.}\quad&
\mathbf u_i\in\mathcal U_{c_i}(\mathbf x_i),
\qquad \forall i,
\\
&
\dot h_{\theta,i}(\mathcal G,\mathbf u)
+
\alpha(h_{\theta,i})
\ge0,
\quad \forall i .
\end{aligned}
\label{eq:training_qp}
\end{equation}
}
where $\mathbf u_i^{\mathrm{ref}}$ is the nominal control.
The barrier derivative is evaluated in the shared operational space
under each robot's model-specific dynamics, so the CBF constraints
retain the corresponding
$\boldsymbol{\eta}_{c_i}(\mathbf x_i)$,
$\mathbf G_{c_i}(\mathbf x_i)$, and
$\mathcal U_{c_i}(\mathbf x_i)$.

The barrier loss is
$\mathcal L_{\mathrm{CBF}}
=
\mathcal L_{\mathrm{unsafe}}
+
\mathcal L_{\mathrm{safe}}
+
\lambda_{\dot h}\mathcal L_{\mathrm{der}}$.
The first two terms use margin penalties
$[h_\theta+\epsilon]_+$ and
$[-h_\theta+\epsilon]_+$ to encourage negative and positive barrier
values for proactively labeled unsafe and safe samples, respectively.
The derivative loss penalizes CBF violations using the one-step estimate
$\widehat{\dot h}_{\theta,i}
=
\big(h_\theta(\mathcal G_i^+)-h_\theta(\mathcal G_i)\big)/\Delta t$.
The successor graph $\mathcal G_i^+$ is reconstructed after propagating
the model-specific dynamics under the current policy, updating the
operational state and state-dependent features, including the RCF.
The policy is supervised by the CBF-QP action through
$\mathcal L_{\mathrm{ctrl}}
=
|\mathcal B|^{-1}
\sum_{i\in\mathcal B}
\|\mathbf u_i^\pi-\mathbf u_i^{\mathrm{qp}}\|_2^2$,
where $\mathcal B$ denotes the training minibatch.
The overall objective is
$\mathcal L
=
\mathcal L_{\mathrm{CBF}}
+
\lambda_a\mathcal L_{\mathrm{ctrl}}$.

The CBF-QP is used only during training. At deployment, robot $i$
constructs $\mathcal G_i$ from local observations and applies
$\pi_\phi(\mathcal G_i)$ through its model-specific control interface,
without online QP optimization or global team information.

\begin{figure*}[t]
    \centering
    \captionsetup[subfigure]{skip=0pt}

    \begin{subfigure}[t]{0.32\textwidth}
        \vspace{0pt}
        \centering

        \begin{minipage}[t][3.25cm][t]{\linewidth}
            \vspace{0pt}
            \centering

            \includegraphics[
                width=0.98\linewidth,
                height=0.36cm,
                keepaspectratio
            ]{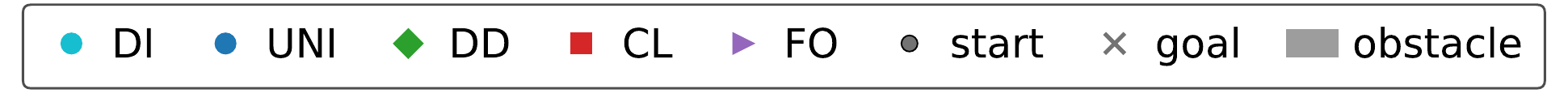}

            \vspace{0.15em}

            \includegraphics[
                width=\linewidth,
                height=3.45cm,
                keepaspectratio
            ]{Figures/test_rollout_ex.pdf}
        \end{minipage}

        \caption{Evaluation settings}
        \label{fig:test_env}
    \end{subfigure}
    \hfill
    \begin{subfigure}[t]{0.32\textwidth}
        \vspace{0pt}
        \centering

        \begin{minipage}[t][3.25cm][t]{\linewidth}
            \vspace{0pt}
            \centering

            \includegraphics[
                width=0.92\linewidth,
                height=0.36cm,
                keepaspectratio
            ]{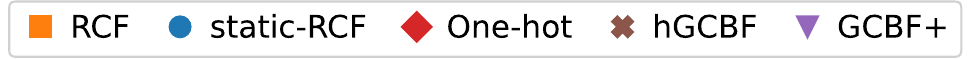}

            \vspace{0.15em}

            \includegraphics[
                width=\linewidth,
                height=1.62cm,
                keepaspectratio
            ]{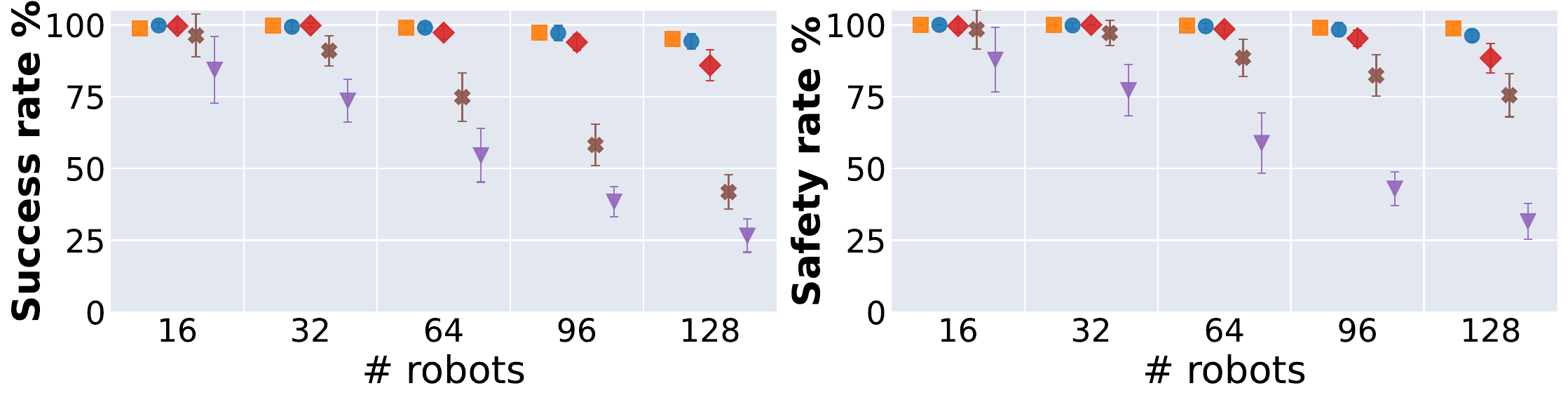}

            \vspace{0.05em}

            \includegraphics[
                width=\linewidth,
                height=1.62cm,
                keepaspectratio
            ]{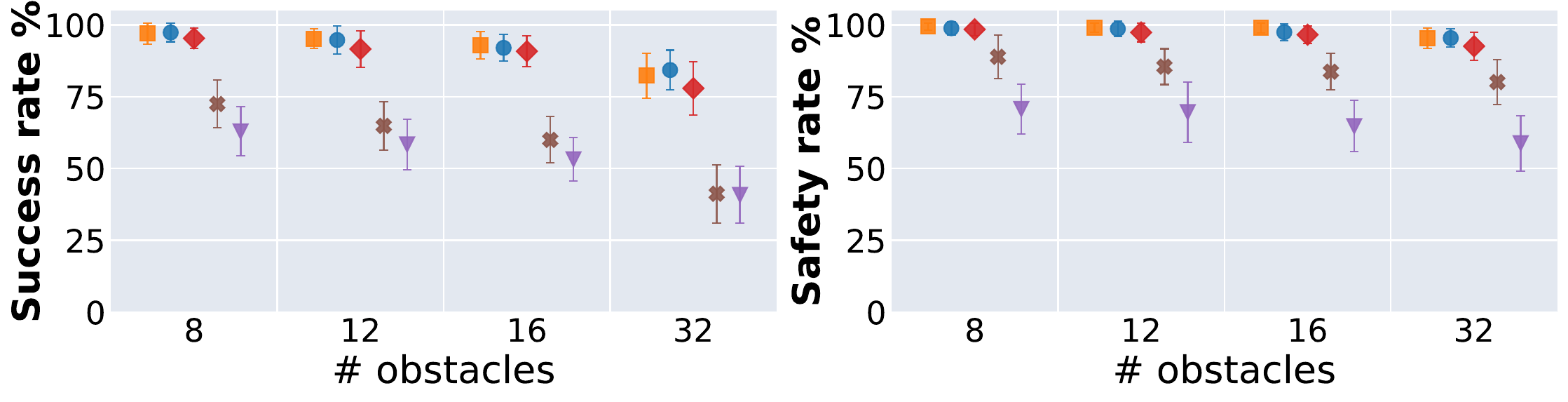}
        \end{minipage}

        \caption{Scalability with seen models}
        \label{fig:main_experiments}
    \end{subfigure}
    \hfill
    \begin{subfigure}[t]{0.32\textwidth}
        \vspace{0pt}
        \centering

        \begin{minipage}[t][3.25cm][t]{\linewidth}
            \vspace{0pt}
            \centering

            \includegraphics[
                width=0.92\linewidth,
                height=0.36cm,
                keepaspectratio
            ]{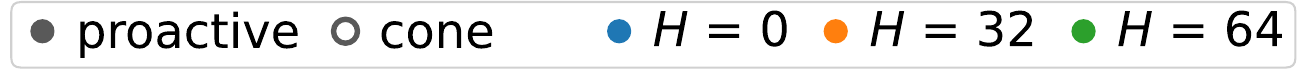}

            \vspace{0.15em}

            \includegraphics[
                width=\linewidth,
                height=1.62cm,
                keepaspectratio
            ]{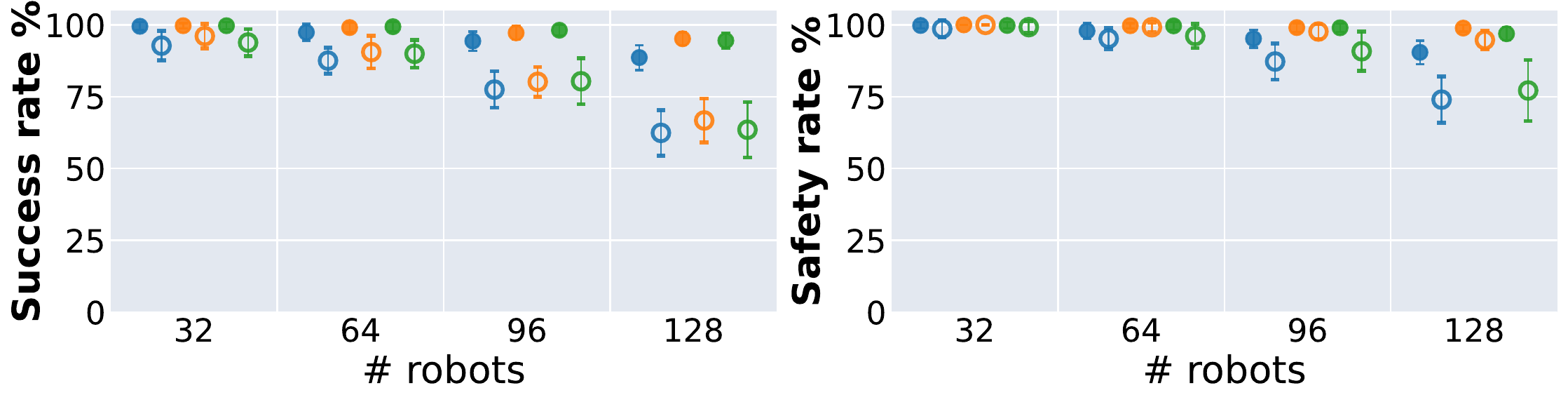}

            \vspace{0.05em}

            \includegraphics[
                width=\linewidth,
                height=1.62cm,
                keepaspectratio
            ]{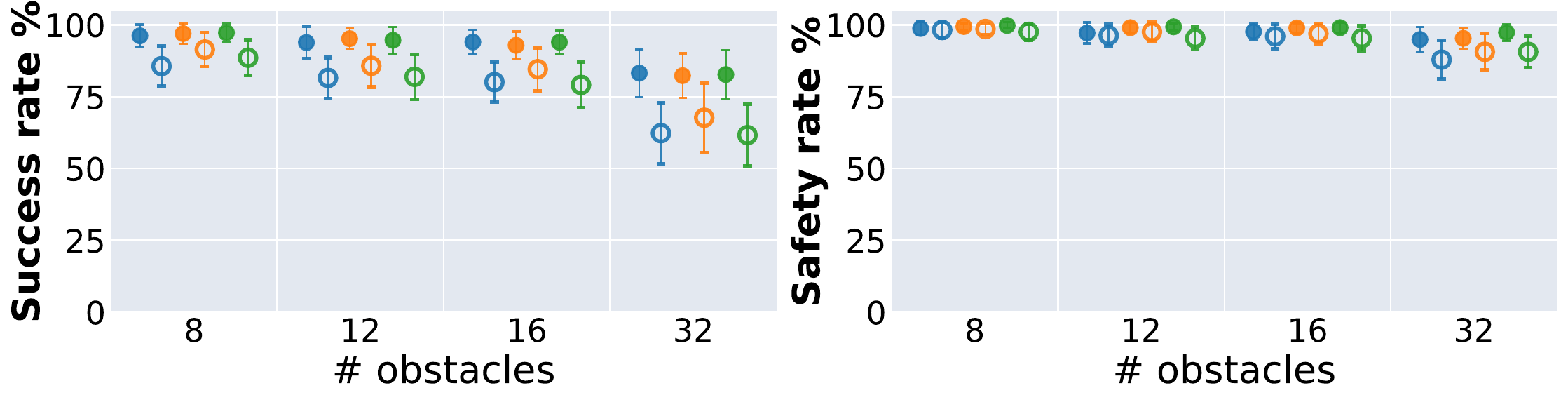}
        \end{minipage}

        \caption{Proactive-labeling ablation}
        \label{fig:horizon_sweep}
    \end{subfigure}

    \captionsetup{skip=0.3em}
    \caption{
    Experimental settings and results.
    (a) Random start--goal evaluation settings in a
    $22\,\mathrm{m}\times22\,\mathrm{m}$ workspace with $N=128$
    robots ($32$ per model) in \texttt{Free}, and $N=32$ robots
    ($8$ per model) with $N_{\mathrm{obs}}=32$ static obstacles in
    \texttt{Obs}. Colored markers denote DI, UNI, DD, and CL robots.
    (b) Scalable heterogeneous navigation with seen robot models.
    In \texttt{Free}, performance is evaluated as team size increases;
    in \texttt{Obs}, performance is evaluated as obstacle density
    increases with $N=32$.
    (c) Effect of proactive safety labeling and safe-label horizon $H$
    compared with velocity-cone-based labeling. Results are shown under
    increasing team size in \texttt{Free} and increasing obstacle density
    in \texttt{Obs}.
    }
    \label{fig:combined_experiments}
    \vspace{-1.5em}
\end{figure*}

\section{Experiments}
\label{sec:experiments}

We evaluate CC-HGCBF in terms of (i) scalability to larger
heterogeneous teams and denser environments, (ii) zero-shot
generalization to unseen robot models and capability variations, and
(iii) the effects of proactive safety labeling and RCF design choices\footnote{\noindent\textit{{Implementation details:}}
For the main experiments, we use proactive safety-label horizon
$H=32$, labeling gain $\lambda=2.0$, and RCF Fourier order {$K=8$}.
The barrier and policy networks each use one message-passing layer with
128-dimensional messages and $(256,256)$ MLP hidden layers.
The supervisory CBF-QP is solved using
ProxQP~\cite{bambade2022proxQP,so2024jaxproxqp}, and all methods are
implemented in JAX~\cite{jax2018github}.}.


\noindent\textit{\textbf{Heterogeneous mobile robot models.}}
We consider four representative mobile-robot models with distinct
motion constraints: a holonomic double-integrator (DI) and
nonholonomic unicycle (UNI), differential-drive (DD), and car-like
(CL) models~\cite{lynch2017modern_robotics}. All models admit the
shared operational representation in
Sec.~\ref{subsec:operational_space}, while retaining their
model-specific feasible-velocity sets:
\begin{table}[h]
\centering
\label{tab:robot_models}
\footnotesize
\setlength{\tabcolsep}{5pt}
\renewcommand{\arraystretch}{0.99}
\resizebox{0.95\columnwidth}{!}{%
\begin{tabular}{@{}cccc@{}}
\toprule
DI & UNI & DD & CL \\
\midrule
$\begin{aligned}
|v_x|&\le v_{\max}\\
|v_y|&\le v_{\max}
\end{aligned}$
&
$\begin{aligned}
|v|&\le v_{\max}\\
|\omega|&\le \omega_{\max}
\end{aligned}$
&
$\displaystyle
\frac{|v|}{v_{\max}}
+
\frac{|\omega|}{\omega_{\max}}
\le 1$
&
$\displaystyle
|\omega|
\le
\frac{|v|\tan(\phi_{\max})}{l}$
\\
\bottomrule
\end{tabular}%
} \vspace{-1em}
\end{table}
\noindent These models span motion capabilities ranging from holonomic
omnidirectional motion to steering-constrained nonholonomic motion.

\noindent\textit{\textbf{Training and evaluation.}}
Unless otherwise specified, all methods are trained with $N=8$ robots (two each of DI, UNI, DD, and CL) in a
$10\,\mathrm{m}\times10\,\mathrm{m}$ workspace with randomly sampled start--goal navigation scenarios, and evaluated in a
$22\,\mathrm{m}\times22\,\mathrm{m}$ workspace (Fig.~\ref{fig:test_env}).
In \texttt{Free}, the team size increases
from $N=16$ to $128$ without static obstacles. In \texttt{Obs},
$N=32$ is fixed while the number of obstacles increases from
$N_{\mathrm{obs}}=8$ to $32$. Mixed teams contain equal numbers of
the four robot models unless otherwise specified.

For zero-shot evaluation, the training models are modified as described
in Sec.~\ref{subsec:zero_shot_generalization}. For unseen CL, CL is
excluded from training. For unseen FO, training uses DI, UNI, DD, and
CL, while forward-only (FO) robots are introduced only at evaluation.
FO follows the CL steering constraint but additionally satisfies
$0\leq v\leq v_{\max}$, producing an asymmetric feasible-velocity set
absent from training. We also evaluate unseen variations of the CL
steering limit.

Each configuration is evaluated on 32 fixed randomly sampled
start--goal scenarios shared across methods, with a maximum of
$3000$ simulation steps and $\Delta t=0.03\,\mathrm{s}$.
All methods use the same decentralized APF nominal
controller~\cite{khatib1986apf}.
We report \emph{safety rate}, the percentage of robots completing the
rollout without any robot--robot or robot--obstacle separation
violation, and \emph{success rate}, the percentage that remain
collision-free and reach their assigned goals within the rollout
horizon.

\noindent\textit{\textbf{Baselines.}}
We compare RCF-based CC-HGCBF with two graph-CBF baselines and two
representation variants.
\textbf{GCBF+}~\cite{zhang2025gcbf+} is adapted to the shared
operational representation but receives no robot-type or capability
features.
\textbf{hGCBF}~\cite{yang2026hgcbf} uses zero-padding for feature
alignment and predefined robot-type information.
\textbf{One-hot} replaces RCF with a discrete robot-model encoding,
while \textbf{static-RCF} uses the full model-specific feasible-velocity
set without state dependence.
One-hot isolates capability conditioning from discrete type
conditioning, while static-RCF isolates the effect of state dependence.
All methods otherwise use the same training and evaluation protocol.
\vspace{-0.3em}

\subsection{Safe and Scalable Heterogeneous Navigation}
\label{subsec:exp_safe_scale}

We first evaluate navigation using the four robot models observed during
training (DI, UNI, DD, and CL). As shown in
Fig.~\ref{fig:main_experiments}, the CC-HGCBF variants consistently
outperform GCBF+ and hGCBF as team size and obstacle density increase.
In \texttt{Free} with $N=128$, RCF achieves $95.2\%$ success and
$98.9\%$ safety, while both baselines degrade substantially,
particularly in success rate. The same trend holds in the most
cluttered \texttt{Obs} setting ($N_{\mathrm{obs}}=32$), where RCF
maintains high safety and goal-reaching performance despite increased
interactions. Differences among RCF, static-RCF, and One-hot remain
comparatively small when all robot models are observed during training.
These results indicate that the shared operational representation alone
is insufficient for scalable heterogeneous coordination, while
robot-specific heterogeneity information becomes increasingly beneficial
as interaction density grows. The relatively small differences among
RCF, static-RCF, and One-hot in this seen-model setting motivate the
zero-shot evaluations below.

\subsection{Zero-Shot Generalization to Unseen Capabilities}
\label{subsec:zero_shot_generalization}


We evaluate zero-shot generalization to robot models and motion
capabilities unseen during training: CL, whose steering capability is
strongly speed-dependent; forward-only (FO), whose feasible motion is
asymmetric; and unseen CL steering limits. Policies are evaluated
without retraining or test-time adaptation. We consider only
heterogeneous-team methods. For discrete-identity baselines, the
held-out model's identity feature is inactive during training and
activated only at evaluation

\subsubsection{Generalization to Unseen Robot Models}
\label{subsubsec:unseen_model}

We first hold out CL during training and train all methods using only
DI, UNI, and DD. Unlike these training models, CL has a strongly
speed-dependent steering constraint, causing its turning capability to
vary substantially with longitudinal velocity. As shown in
Fig.~\ref{fig:heldout_generalization}
(\subref{fig:heldout_cl_all_free})--(\subref{fig:heldout_cl_all_obs}),
RCF maintains substantially higher success and safety as team size and
obstacle density increase, with larger gains on the held-out CL robots
(\subref{fig:heldout_cl_only_free})--(\subref{fig:heldout_cl_only_obs}).
At $N=128$ in \texttt{Free}, RCF improves overall success and safety
over static-RCF by $10.8$ and $13.1$, respectively;
for the held-out CL robots, the gains increase to $14.7$ and $23.4$\%. RCF further outperforms the type-conditioned One-hot and hGCBF
baselines by $32.5$--$32.7$\% in success and $41.7$--$41.9$\%
in safety on the held-out CL robots.

We next introduce FO robots only at evaluation after training on DI,
UNI, DD, and CL. FO retains the steering constraint of CL but further
restricts longitudinal velocity $\in[0,v_{\max}]$,
producing a
forward--backward asymmetric capability absent from training. As shown
in Fig.~\ref{fig:heldout_generalization}
\subref{fig:heldout_fo_free} and \subref{fig:heldout_fo_obs16},
RCF maintains high safety on the unseen FO robots. In the most
challenging obstacle setting with $N=30$, RCF improves safety over
static-RCF by $15.6$\% while maintaining comparable
success. Together with the held-out CL results, this demonstrates
zero-shot transfer to qualitatively different unseen motion
capabilities, not only unseen robot identities.

\subsubsection{Generalization to Parametric Capability Variations}
\label{subsubsec:parametric_generalization}

We further evaluate parametric capability generalization by varying the
maximum steering angle of CL. The policy is trained only at
$\phi_{\max}=40^\circ$ and evaluated without retraining at
$\phi_{\max}\in\{30^\circ,40^\circ,50^\circ\}$.
As shown in Table~\ref{tab:rcf_phimax_scaling}, RCF remains robust to
unseen steering limits: relative to the training setting, success
changes by at most $2.9$\% and safety by at most $1.5$\% across all evaluated team sizes and obstacle densities.
Together with the held-out CL results, these results demonstrate
zero-shot transfer across unseen velocity-constraint geometries. \vspace{-3pt}

\begin{table}[h]
\centering
\caption{RCF generalization to maximum steering-angle variations across team sizes and obstacle densities. Trained only at $\phi_{\max}=40^\circ$ ($\dagger$), the policy is evaluated without retraining. Entries report Success\,/\,Safety (\%).}
\label{tab:rcf_phimax_scaling}
\footnotesize
\setlength{\tabcolsep}{4.0pt}
\renewcommand{\arraystretch}{0.7}

\begin{tabular}{ccccc}
\toprule
\multirow{2}{*}{Env.}
& \multirow{2}{*}{Setting}
& \multicolumn{3}{c}{Maximum steering angle, $\phi_{\max}$} \\
\cmidrule(lr){3-5}
& & $30^\circ$ & $40^\circ{}^{\dagger}$ & $50^\circ$ \\
\midrule

\multirow{5}{*}{\texttt{Free}}
& $N=16$  & 99.0 / 100.0 & 98.8 / 100.0 & 99.8 / 100.0 \\
& $N=32$  & 99.1 / 100.0 & 99.7 / 100.0 & 99.9 / 100.0 \\
& $N=64$  & 98.0 / 99.4  & 99.1 / 99.7  & 99.5 / 100.0 \\
& $N=96$  & 96.0 / 99.0  & 97.2 / 99.1  & 98.6 / 99.9 \\
& $N=128$ & 92.5 / 97.6  & 95.2 / 98.9  & 97.3 / 99.2 \\
\midrule

\multirow{4}{*}{\texttt{Obs}}
& $N_{\mathrm{obs}}=8$  & 95.5 / 98.6 & 97.0 / 99.4 & 98.1 / 99.8 \\
& $N_{\mathrm{obs}}=12$ & 92.5 / 97.5 & 95.2 / 99.0 & 96.0 / 98.8 \\
& $N_{\mathrm{obs}}=16$ & 91.4 / 97.8 & 92.9 / 98.9 & 94.5 / 98.6 \\
& $N_{\mathrm{obs}}=32$ & 79.4 / 96.4 & 82.3 / 95.3 & 85.2 / 96.5 \\
\bottomrule
\end{tabular}
\vspace{-1.5em}
\end{table}

\subsection{Ablation and Analysis}
\label{subsec:ablation}








\subsubsection{Proactive Safety Labeling and Horizon Sensitivity}
\label{subsubsec:labeling_ablation}

We compare the proactive safety labeling in
Sec.~\ref{subsec:safety_labeling} with the velocity-cone-based rule
used in the GCBF+ implementation~\cite{gcbfplus_code},
while keeping RCF and all other training settings fixed.
The velocity-cone rule identifies potentially unsafe interactions using
a prescribed warning region and collision cone, whereas our labeling
uses separation and relative closing velocity without an explicit
warning region.

Figure~\ref{fig:horizon_sweep} compares the two labeling schemes for
$H\in\{0,32,64\}$. In \texttt{Free} with $N=128$, proactive labeling
improves success by $26.2$, $28.5$, and $31.0$ \% for
$H=0$, $32$, and $64$, respectively. With
$N_{\mathrm{obs}}=32$, the corresponding gains are $20.9$, $14.6$,
and $21.1$ points, while safety is consistently higher across all
horizons. These results show that the benefit of proactive labeling is
robust to the choice of $H$, with substantial improvements in
goal-reaching performance under dense interactions.

\subsubsection{Sensitivity to Fourier Order}
\label{subsubsec:k_ablation}

We further evaluate the sensitivity of RCF to the Fourier order $K$.
Table~\ref{tab:k_ablation} reports performance on seen-model and
held-out FO and CL evaluations, averaged over five training seeds.
Performance varies modestly across $K$: the maximum success-rate
variation is $0.7$ \% for held-out FO and $2.1$ points
for held-out CL, with no consistent benefit from increasing $K$.
We therefore use {$K=8$} in the main experiments to balance directional
resolution and representation size. \vspace{-6pt}

\begin{table}[h]
\centering
\caption{
Ablation on Fourier order $K$.
Performance (\%) is averaged over the corresponding evaluation
configurations and five training seeds; $\Delta_K$ denotes the range
across $K$.
}
\label{tab:k_ablation}
\footnotesize
\setlength{\tabcolsep}{4.5pt}
\renewcommand{\arraystretch}{1.08}
\begin{tabular}{llccccc}
\toprule
Evaluation & Metric
& $K=4$ & {$K=8$} & $K=12$ & $K=16$
& $\Delta_K$ \\
\midrule
\multirow{2}{*}{Seen models}
& Success & 94.8 & 94.5 & 94.7 & 94.9 & 0.4 \\
& Safety  & 98.2 & 97.7 & 97.7 & 97.5 & 0.7 \\
\midrule
\multirow{2}{*}{Held-out FO}
& Success & 94.5 & 95.0 & 95.0 & 95.2 & 0.7 \\
& Safety  & 98.3 & 98.2 & 98.1 & 97.7 & 0.6 \\
\midrule
\multirow{2}{*}{Held-out CL}
& Success & 82.4 & 80.9 & 80.6 & 82.7 & 2.1 \\
& Safety  & 85.1 & 86.1 & 86.1 & 86.0 & 1.0 \\
\bottomrule
\end{tabular}
\vspace{-1.5em}
\end{table}

\begin{figure*}[!tp]
    \centering

    \includegraphics[width=0.5\textwidth]{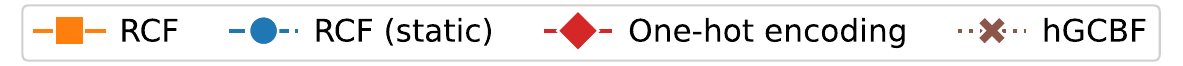}
    \par\vspace{0.3em}

    \begin{subfigure}[t]{0.34\textwidth}
        \centering
        \includegraphics[width=\linewidth]
        {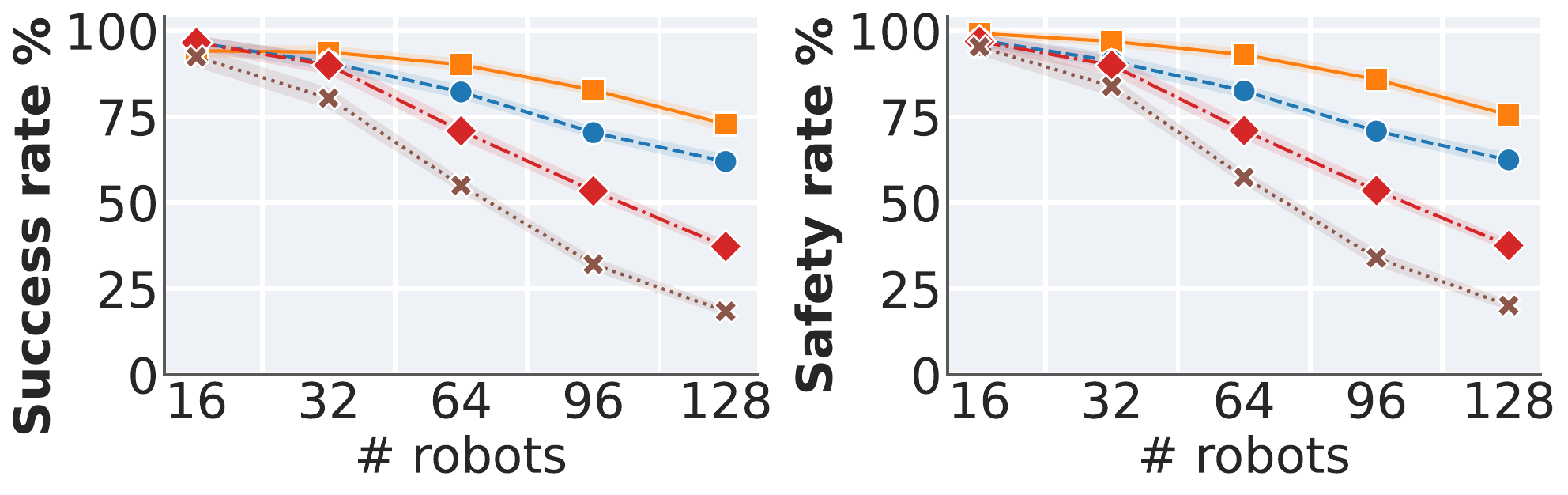}
        \caption{All agents, held-out CL, \texttt{Free}}
        \label{fig:heldout_cl_all_free}
    \end{subfigure}%
    \hfill%
    \begin{subfigure}[t]{0.34\textwidth}
        \centering
        \includegraphics[width=\linewidth]
        {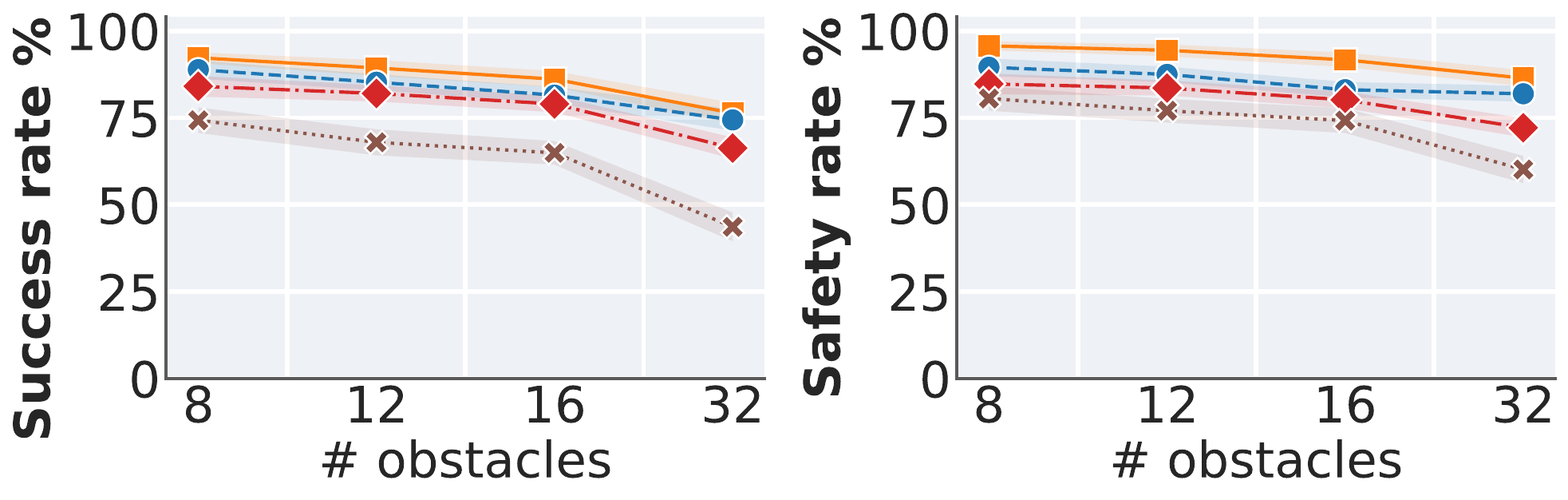}
        \caption{All agents, held-out CL, \texttt{Obs}}
        \label{fig:heldout_cl_all_obs}
    \end{subfigure}%
    \hfill%
    \begin{subfigure}[t]{0.29\textwidth}
        \centering
        \includegraphics[width=\linewidth]
        {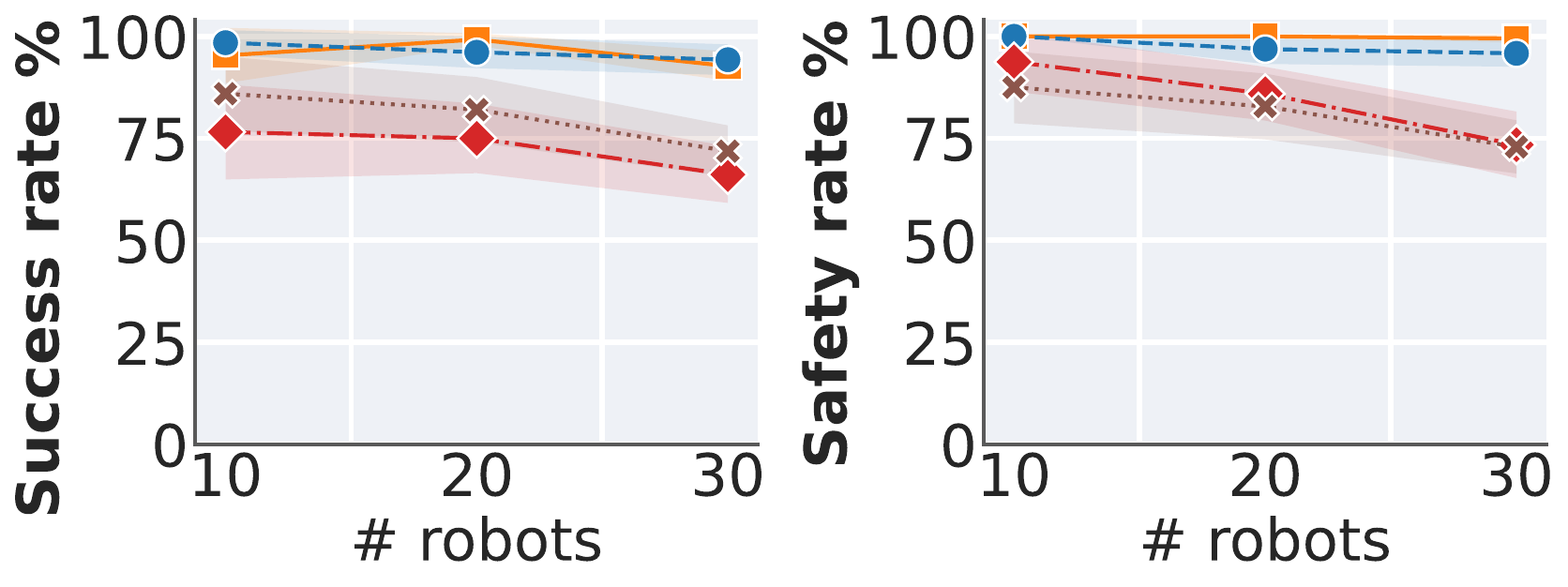}
        \caption{Held-out FO only, \texttt{Free}}
        \label{fig:heldout_fo_free}
    \end{subfigure}

    \vspace{0.1em}

    \begin{subfigure}[t]{0.34\textwidth}
        \centering
        \includegraphics[width=\linewidth]
        {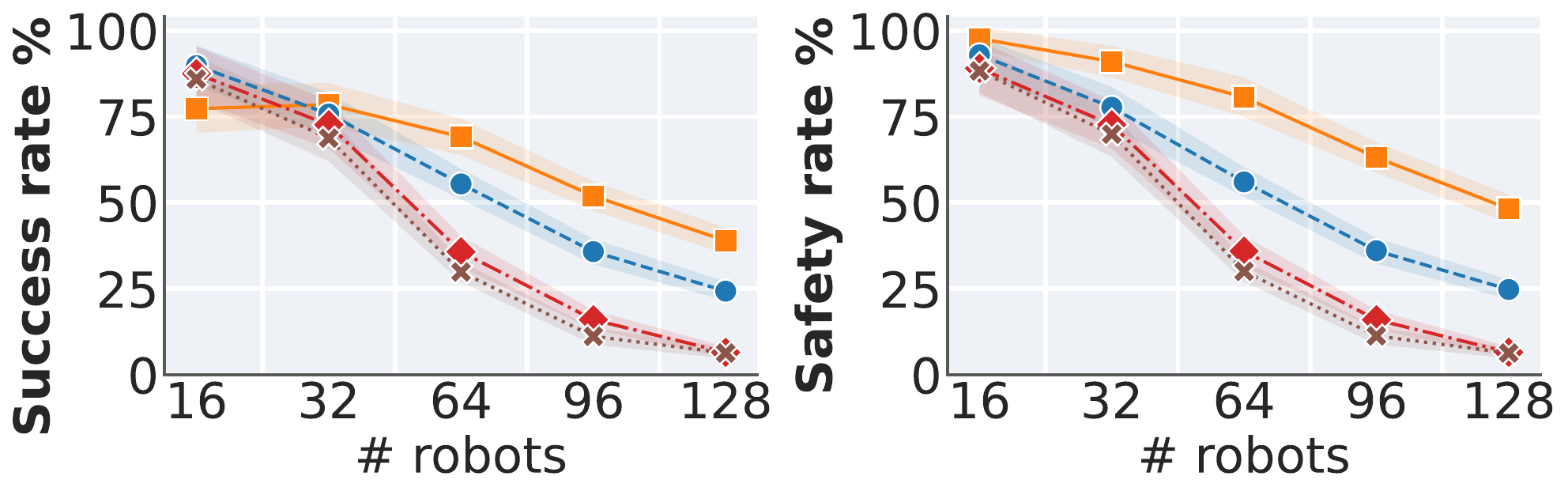}
        \caption{Held-out CL only, \texttt{Free}}
        \label{fig:heldout_cl_only_free}
    \end{subfigure}%
    \hfill%
    \begin{subfigure}[t]{0.34\textwidth}
        \centering
        \includegraphics[width=\linewidth]
        {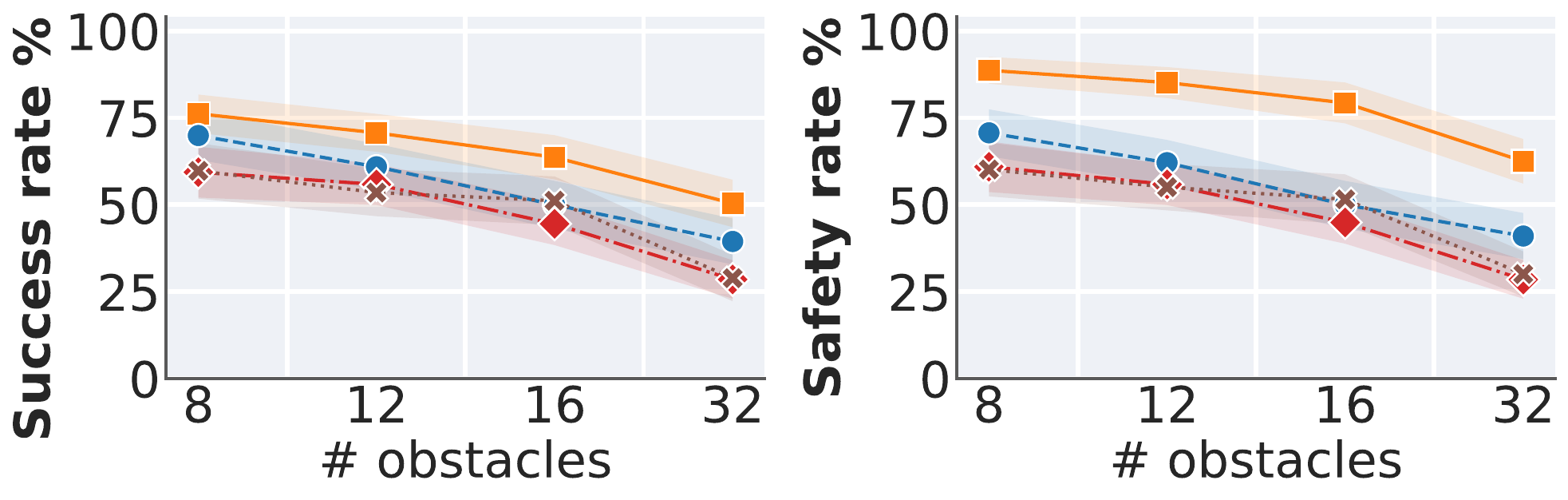}
        \caption{Held-out CL only, \texttt{Obs}}
        \label{fig:heldout_cl_only_obs}
    \end{subfigure}%
    \hfill%
    \begin{subfigure}[t]{0.29\textwidth}
        \centering
        \includegraphics[width=\linewidth]
        {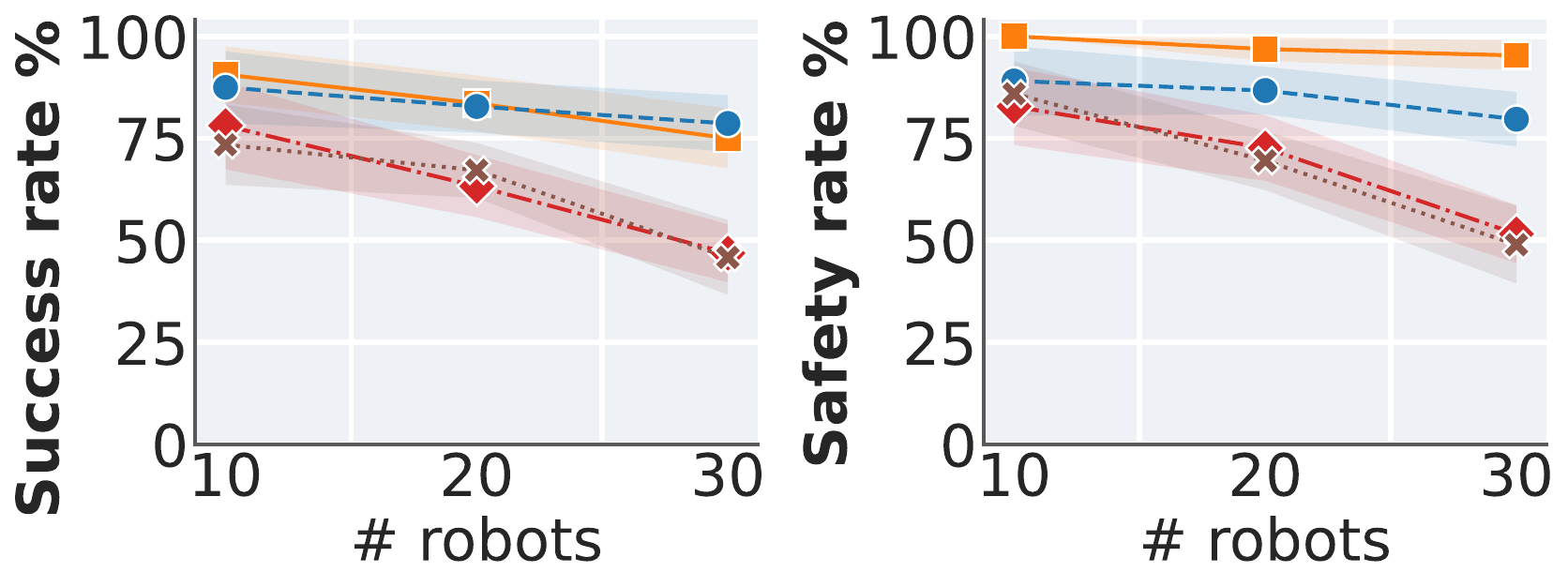}
        \caption{Held-out FO only, 16 obstacles}
        \label{fig:heldout_fo_obs16}
    \end{subfigure}

    \caption{
    Zero-shot generalization to unseen CL and FO models.
    For unseen CL, all methods are trained on DI, UNI, and DD;
    \subref{fig:heldout_cl_all_free}--\subref{fig:heldout_cl_all_obs}
    report performance over all agents, whereas
    \subref{fig:heldout_cl_only_free}--\subref{fig:heldout_cl_only_obs}
    report performance for the unseen CL agents only.
    For unseen FO, methods are trained on DI, UNI, DD, and CL;
    \subref{fig:heldout_fo_free} and \subref{fig:heldout_fo_obs16}
    report performance for the unseen FO agents only. Shaded regions denote variance.
    }
    \label{fig:heldout_generalization}
    \vspace{-1.8em}
\end{figure*}

\subsection{Physical Robot Validation}
\label{phys_robot}

We further validate zero-shot deployment of CC-HGCBF on five DJI
RoboMaster S1 robots, with model-specific kinematic and control
constraints enforced through their respective control interfaces.
We consider unseen CL and FO settings, each with and without static
obstacles, and compare RCF with One-hot and hGCBF.
Figure~\ref{fig:qualitative} shows the unseen FO setting with obstacles:
RCF produces collision-free trajectories, whereas hGCBF and One-hot
result in robot--obstacle and robot--robot collisions, respectively.
These results demonstrate zero-shot physical deployment to unseen motion
capabilities without retraining.
\vspace{-3pt}















\section{Conclusion}
\label{sec:conclusion}

We presented CC-HGCBF, a capability-conditioned graph CBF framework for
safe decentralized navigation of heterogeneous mobile-robot teams. By
combining a shared second-order operational-space representation with the
state-dependent RCF, the framework enables a single graph barrier and
policy to reason about heterogeneous motion capabilities without relying
on predefined robot identities. Experiments demonstrate scalable
coordination and zero-shot transfer to unseen robot models and capability
variations, with additional validation on physical robots.
Future work will extend the framework to more general robotic systems
and operational spaces, investigate guarantees for deadlock-free
coordination, and develop formal safety guarantees for learned policies
under finite control-update intervals.

%
\bibliographystyle{IEEEtran} \footnotesize
\bibliography{references}

\end{document}